\def\identifiedcopy{1}
\documentclass{article}
\usepackage[T1]{fontenc}
\usepackage{iclr2027_conference,times}
\usepackage{amsmath,amssymb,amsthm,booktabs,tabularx,graphicx,float,needspace,tikz}
\usepackage{hyperref,url}
\hypersetup{hidelinks,pdftitle={Cost-Efficient LLM Evaluation: A Framework for Narrow Confidence Intervals Under Fixed Budgets}}
\ifdefined\identifiedcopy
\hypersetup{pdfauthor={Yezhou Cheng, Runjia Du, Zeming Liu, Qibai Chen, Hang Lyu, Yilan Wei, Yankai Zeng, Bojun Lin}}
\else
\hypersetup{pdfauthor={}}
\fi
\newcommand{\E}{\mathbb{E}}
\newcommand{\Prb}{\mathbb{P}}
\newcommand{\Var}{\operatorname{Var}}
\newcommand{\ind}{\mathbf{1}}

\newcommand{\Bin}{\operatorname{Bin}}
\newtheorem{theorem}{Theorem}
\newtheorem{proposition}[theorem]{Proposition}

\title{Cost-Efficient LLM Evaluation:\\A Framework for Narrow Confidence\\Intervals Under Fixed Budgets}
\author{%
\makebox[\dimexpr\textwidth-2\tabcolsep\relax][c]{%
\begin{tabular}[t]{@{}c@{}}
\href{https://openreview.net/profile?id=\%7EYezhou_Cheng1}{Yezhou Cheng}$^{1}$ \quad
\href{https://openreview.net/profile?id=\%7ERunjia_Du1}{Runjia Du}$^{2}$ \quad
\href{https://openreview.net/profile?id=\%7EZeming_Liu5}{Zeming Liu}$^{3}$ \quad
\href{https://openreview.net/profile?id=\%7EQibai_Chen1}{Qibai Chen}$^{4}$ \\
\href{https://openreview.net/profile?id=\%7EHang_Lyu2}{Hang Lyu}$^{3}$ \quad
\href{https://openreview.net/profile?id=\%7EYilan_Wei1}{Yilan Wei}$^{5}$ \quad
\href{https://openreview.net/profile?id=\%7EYankai_Zeng1}{Yankai Zeng}$^{6}$ \quad
\href{https://openreview.net/profile?id=\%7EBojun_Lin1}{Bojun Lin}$^{7}$ \\[0.4em]
\normalfont $^{1}$University of Wisconsin--Madison \quad
$^{2}$Purdue University \\
\normalfont $^{3}$Brown University \quad
$^{4}$Carnegie Mellon University \\
\normalfont $^{5}$\href{https://u.northwestern.edu}{Northwestern University} \quad
$^{6}$\href{https://independent-researcher.org}{Independent} \quad
$^{7}$\href{https://www.pinterest.com}{Pinterest, Inc.}
\end{tabular}%
}%
}
\ifdefined\identifiedcopy\iclrfinalcopy\fi
\begin{document}
\maketitle
\lhead{}
\renewcommand{\headrulewidth}{0pt}

\vspace{-\baselineskip}
\begin{abstract}
Large language models (LLMs) play an increasingly important role in both
academic research and industry. Reliable evaluation helps researchers
understand model capabilities and assess progress during model development.
However, evaluation can be costly, and researchers often need to assess
performance within a fixed budget. Because model outputs can vary across
runs, this budget must be allocated between testing different tasks and
repeating selected tasks. A confidence interval quantifies the uncertainty
in the resulting performance estimate. Obtaining narrower intervals at the
same confidence level and evaluation cost provides a more precise assessment
of model performance, helping researchers draw clearer conclusions from
limited observations. We introduce a framework that allocates a fixed
evaluation budget across tasks and repeated runs, with or without full task
coverage. The framework then uses disagreements between runs on the same
task to bound estimation error, allowing narrower confidence intervals at
the same confidence level when the sampled repeats show little variability.
When repeated runs on each task yield the same outcome, we prove that our
framework achieves the fastest theoretically possible rates for narrowing
confidence intervals in the worst case as the number of tasks and the budget
for repeats grow. We establish matching lower bounds both when every task
must be covered and when tasks may be skipped, even for evaluators who use
earlier results to decide which run to carry out next. On LiveCodeBench, our
framework yielded narrower confidence intervals on average than uniform
random sampling in 15 of 16 model configurations, with a median width
reduction of 30.6\% at the same confidence level and the same budget of
1,100 outputs per configuration.
\end{abstract}
\vspace{-\baselineskip}

\section{Introduction}

Large language model (LLM) evaluation helps researchers decide whether a model
is ready for further testing or deployment. Outputs can vary across runs of
the same task, and repeated sampling is often used to assess whether a model
can solve a task within several attempts \citep{brown2024,kazdan2025}.
Coding agents, which use language models and tools to complete programming
tasks, also vary across executions \citep{bjarnason2026}. Under a fixed
evaluation budget, every additional run on one task uses resources that could
cover another task. This allocation affects both the estimated benchmark
score and the uncertainty in that estimate.

A confidence interval expresses this uncertainty as a range of possible
values for the score. For example, an evaluator may require a model's success
rate to exceed a threshold before advancing it to the next evaluation stage.
An interval entirely above the threshold supports that decision, whereas an
interval crossing it leaves the decision unresolved. Narrower intervals at
the same confidence level can resolve more decisions at the same cost.

Our starting point is that getting the right score and knowing how precise
it is can require different observations. Consider 880 coding tasks with
five planned runs per task. We specify the tasks, model version, generation
settings, and seeds that fix the random choices. The target is the average success
rate across all 4,400 task--run pairs, which we call the \emph{fixed cohort}.
If all five outcomes agree within each task, one run per task gives the
exact score. With one observation per task, the evaluator cannot tell
whether the other four runs would give the same result.

We spend part of the budget on \emph{replication}, meaning repeated runs.
With 1,100 outputs and full task coverage, we randomly select one run for
each of the 880 tasks and a second distinct run for 220 randomly selected
tasks. We choose all runs before seeing any outcomes. We average
within tasks and then across tasks so each task has equal weight. We also
count \emph{disagreements}, where two observed runs of the same task have
different outcomes. Our \emph{Joint} interval uses the estimated score and
this count. Few observed disagreements can support a smaller bound on the error
from unobserved runs.

When full task coverage is optional, we can randomly sample a subset of
tasks and spend the saved budget on more repeats. The interval then accounts
for uncertainty from skipped tasks as well as unobserved runs. A run can
also consist of several model responses or agent executions. In that case,
the budget counts individual responses or executions and must hold even
when every run reaches its maximum length.

Interval coverage is a guarantee about the score. For example, a 95\%
interval contains the full planned score at least 95\% of the time over
repeated random selections of runs from the same fixed set of outcomes.
We call an interval \emph{honest} if this guarantee holds for every such
set, including when outcomes differ or are dependent.

We make three contributions.
\begin{itemize}
\item We give a fixed-budget evaluation framework that chooses which tasks
and runs to observe and uses observed disagreements to construct an honest
confidence interval for the full planned score.
\item Our \emph{sharp replication laws} give the best rates, up to constant
factors, for shrinking the worst-case average width as tasks and repeats
grow when all planned runs agree within each task. We prove both rates, with and without
full task coverage. The lower bounds allow later runs to depend on earlier results.
\item We replay evaluations by sampling recorded outcomes from LiveCodeBench,
a coding benchmark. At 1,100 outputs and the same 95\% confidence level,
Joint with full task coverage gives narrower intervals on average than
uniform sampling in 15 of 16 configurations, with a median reduction of 30.6\%.
\end{itemize}

Figure~\ref{fig:framework_overview} illustrates the allocation choices and
how repeated outcomes inform the confidence interval.

\begingroup
\setlength{\intextsep}{6pt}
\setlength{\abovecaptionskip}{5pt}
\makeatletter
\let\framework@makecaption\@makecaption
\renewcommand{\@makecaption}[2]{{\small\framework@makecaption{#1}{#2}}}
\makeatother
\begin{figure}[H]
\centering
\begin{tikzpicture}[x=1cm,y=1cm,
  font=\sffamily\fontsize{6.6}{7.6}\selectfont,
  execute at begin node={\hyphenpenalty=10000\exhyphenpenalty=10000},
  every node/.style={inner sep=0pt},
  body/.style={anchor=north west,align=left},
  flowtitle/.style={body,font=\sffamily\bfseries\fontsize{7.3}{8.3}\selectfont},
  note/.style={body,text=black!65,font=\sffamily\fontsize{6.2}{7.2}\selectfont},
  card/.style={rounded corners=2pt,line width=.55pt},
  flow/.style={->,line width=.65pt,>=stealth}]
  \definecolor{evalblue}{RGB}{39,93,168}
  \definecolor{evalteal}{RGB}{22,126,121}
  \definecolor{evalorange}{RGB}{188,108,50}
  \path[use as bounding box] (0,0) rectangle (13.8,6.42);

  \node[anchor=base west,font=\sffamily\bfseries\fontsize{9}{10}\selectfont,text=evalblue]
    at (0,6.16) {Our proposed framework};
  \node[note,anchor=base east] at (13.8,6.16)
    {Allocate a fixed evaluation budget according to whether every task must be represented.};
  \node[flowtitle,text=evalblue] at (0,5.97)
    {When full task coverage is required};

  \foreach \x/\w in {0/3.6,4.65/3.4,9.5/4.3} {
    \draw[card,draw=evalblue,fill=evalblue!3] (\x,2.13) rectangle ++(\w,3.53);
    \draw[card,draw=evalteal,fill=evalteal!3] (\x,.02) rectangle ++(\w,1.72);
  }

  \node[flowtitle,text=evalblue,text width=3.25cm] at (.16,5.54)
    {Step 1\\Allocate the budget};
  \node[body,text width=3.27cm,font=\sffamily\bfseries\fontsize{6.6}{7.6}\selectfont]
    at (.16,4.96) {Test every task; repeat a random subset.};
  \node[note,text width=3.27cm] at (.16,4.43)
    {Choose the runs before observing outcomes.};
  \node[text=black!65] at (2.04,3.83) {Planned runs};
  \foreach \row in {0,...,5} {
    \pgfmathtruncatemacro{\task}{\row+1}
    \node[anchor=east,text=black!65,font=\sffamily\fontsize{6.1}{6.5}\selectfont]
      at (1.25,3.54-.22*\row) {Task \task};
    \foreach \col in {0,...,2} {
      \draw[draw=black!20,fill=white] (1.54+.34*\col,3.62-.22*\row)
        rectangle ++(.16,-.16);
    }
  }
  \foreach \row/\col in {0/0,1/1,1/2,2/2,3/0,4/0,4/1,5/2} {
    \fill[evalblue] (1.54+.34*\col,3.62-.22*\row) rectangle ++(.16,-.16);
  }
  \foreach \row/\task in {1/2,4/5} {
    \draw[draw=evalblue,rounded corners=.7pt] (.48,3.647-.22*\row) rectangle ++(2.36,-.214);
    \node[anchor=east,text=evalblue,font=\sffamily\bfseries\fontsize{6.1}{6.5}\selectfont,
      fill=evalblue!3] at (1.25,3.54-.22*\row) {Task \task};
  }
  \node[note,text=evalblue,text width=3.27cm] at (.16,2.35)
    {Outlined tasks are repeated.};

  \draw[flow,evalblue] (3.67,3.56) -- (4.58,3.56);
  \node[align=center,text=evalblue,font=\sffamily\fontsize{6.1}{7}\selectfont]
    at (4.125,3.99) {Tasks\\+ results};

  \node[flowtitle,text=evalblue,text width=3.08cm] at (4.81,5.54)
    {Step 2\\Summarize results};
  \node[body,text width=3.08cm,font=\sffamily\bfseries\fontsize{6.6}{7.6}\selectfont]
    at (4.81,4.96) {Estimate the mean score};
  \node[body,text width=3.08cm] at (4.81,4.66)
    {Average within each task, then across tasks.};
  \draw[black!15] (4.81,4.10) -- (7.89,4.10);
  \node[body,text width=3.08cm] at (4.81,3.99)
    {Compare results within repeated tasks};
  \foreach \y/\task in {3.33/2,2.96/5} {
    \node[anchor=west,text=evalblue,font=\sffamily\bfseries\fontsize{6.6}{7.6}\selectfont]
      at (4.81,\y) {Task \task};
    \draw[draw=evalblue,fill=white] (5.86,\y) circle (.13);
    \draw[evalblue,line width=.65pt] (5.79,\y) -- (5.84,\y-.05) -- (5.94,\y+.06);
  }
  \draw[draw=evalblue,fill=white] (6.24,3.33) circle (.13);
  \draw[evalblue,line width=.65pt] (6.17,3.33) -- (6.22,3.28) -- (6.32,3.39);
  \draw[draw=evalorange,fill=white] (6.24,2.96) circle (.13);
  \draw[evalorange,line width=.65pt] (6.18,2.90) -- (6.30,3.02);
  \draw[evalorange,line width=.65pt] (6.18,3.02) -- (6.30,2.90);
  \node[anchor=west] at (6.51,3.33) {Agree};
  \node[anchor=west] at (6.51,2.96) {Disagree};
  \node[body,text=evalblue,text width=3.08cm] at (4.81,2.70)
    {Count disagreeing pairs.};
  \node[note] at (4.81,2.41)
    {Each row is one task, run twice.};

  \draw[flow,evalblue] (8.12,3.56) -- (9.43,3.56);
  \node[align=center,text=evalblue,font=\sffamily\fontsize{6.1}{7}\selectfont]
    at (8.775,4.08) {Mean score +\\disagreement\\count};

  \node[flowtitle,text=evalblue,text width=3.98cm] at (9.66,5.54)
    {Step 3\\Construct a confidence interval};
  \node[body,text width=3.98cm,font=\sffamily\bfseries\fontsize{6.6}{7.6}\selectfont]
    at (9.66,4.96) {Use disagreements to bound error from unobserved runs.};
  \node[note,text width=3.98cm] at (9.66,4.42) {At the same confidence level:};
  \node[anchor=west,text=evalorange] at (9.66,3.98) {More disagreements};
  \draw[evalorange,line width=1pt] (12.07,3.98) -- (13.59,3.98);
  \foreach \x in {12.07,13.59} {\draw[evalorange,line width=.7pt] (\x,3.89) -- (\x,4.07);}
  \fill[evalorange] (12.83,3.98) circle (1.3pt);
  \node[anchor=west,text=evalblue] at (9.66,3.60) {Fewer disagreements};
  \draw[evalblue,line width=1pt] (12.43,3.60) -- (13.23,3.60);
  \foreach \x in {12.43,13.23} {\draw[evalblue,line width=.7pt] (\x,3.51) -- (\x,3.69);}
  \fill[evalblue] (12.83,3.60) circle (1.3pt);
  \node[body,text=evalblue,text width=3.98cm] at (9.66,3.33)
    {Consistent repeats can support a narrower interval.};
  \node[note,text width=3.98cm] at (9.66,2.71)
    {Dot: mean estimate.\\Line: confidence interval.};

  \node[flowtitle,text=evalteal] at (0,2.03) {When full task coverage is not required};
  \node[flowtitle,text=evalteal,text width=3.27cm] at (.16,1.61)
    {Step 1\\Allocate the budget};
  \node[body,text width=3.27cm] at (.16,1.08)
    {Randomly sample a subset of tasks. Use the saved budget for more repeats. Choose tasks and runs in advance.};
  \draw[flow,evalteal] (3.67,.55) -- (4.58,.55);
  \node[align=center,text=evalteal,font=\sffamily\fontsize{6.1}{7}\selectfont]
    at (4.125,.97) {Tasks\\+ results};
  \node[flowtitle,text=evalteal,text width=3.08cm] at (4.81,1.61)
    {Step 2\\Summarize results};
  \node[body,text width=3.08cm] at (4.81,1.08)
    {Average within each sampled task, then across sampled tasks. Count disagreeing pairs.};
  \draw[flow,evalteal] (8.12,.55) -- (9.43,.55);
  \node[align=center,text=evalteal,font=\sffamily\fontsize{6.1}{7}\selectfont]
    at (8.775,1.07) {Mean score +\\disagreement\\count};
  \node[flowtitle,text=evalteal,text width=3.98cm] at (9.66,1.61)
    {Step 3\\Construct a confidence interval};
  \node[body,text width=3.98cm] at (9.66,1.08)
    {Account for uncertainty from both unobserved runs and tasks that were not sampled.};
\end{tikzpicture}
\caption{Our evaluation framework under two task-coverage requirements.
Filled cells are observed; empty cells are unobserved. Checks denote success
and crosses denote failure. Both designs choose tasks and runs before observing outcomes.
The upper interval comparison is schematic, holding the mean estimate,
sampling design, and confidence level fixed.
When full task coverage is not required, the number of tasks to skip balances
the information from extra repeats against uncertainty from unobserved tasks.}
\label{fig:framework_overview}
\end{figure}
\endgroup

\Needspace{10\baselineskip}
\section{Related work}

\paragraph{Repeated and adaptive evaluation.}
\citet{brown2024} study how the chance of at least one success changes as
the model gets more attempts, and \citet{kazdan2025} model how the gains
vary by task. \citet{bjarnason2026} analyze
repeated coding-agent runs, and \citet{huang2026} replay partial agent
benchmarks to compare two methods and decide when to defer a decision.
\citet{pilditch2026} use a probability model to decide when uncertainty is
small enough to stop sampling. Our guarantee comes from the random choice of runs
on a fixed grid and requires no model for how outcomes were generated.

\paragraph{Sampling with existing information or known groups.}
Earlier methods use observed results to choose later evaluations or use
predicted scores to reduce the number of runs needed
\citep{zrnic2024,angelopoulos2025,fisch2024,wu2026}.
\citet{yauney2026} show why random samples of benchmark items are useful
baselines. Our setting has a strict total cost limit and requires neither
a trained grader nor an initial set of labeled examples.

Sampling within groups, also called \emph{stratified sampling}, can improve
precision by accounting for differences among groups
\citep{burgess2021,spertus2024}. A single observation per group provides
limited evidence about variation within that group \citep{barabesi2012}.
Bounds that depend on variation and methods that adapt interval width to
the data are well established \citep{maurer2009,cai2004}. We establish the
optimal rates in both the number of tasks and the budget for repeats.

\paragraph{Early stopping and uncertainty during sampling.}
Earlier simulation methods can decide whether to stop a run and adjust
the weights of completed runs to account for that choice \citep{prangle2014}.
\citet{feldman2026} estimate event rates by deciding whether to continue
and control cost on average. Our \emph{hard budget} limits cost for every
set of outcomes and every random selection. We reserve enough to finish
every selected run at its maximum length. Established methods also supply
weights for sampling in stages and intervals that remain valid as more
observations arrive \citep{lumley2024,waudby2020,howard2021,wasserman2020}.

Table~\ref{tab:nearest} summarizes the statistical guarantees of related methods.

\begingroup
\setlength{\intextsep}{7pt}
\setlength{\abovecaptionskip}{3pt}
\setlength{\belowcaptionskip}{4pt}
\begin{table}[H]
\centering
\caption{Positioning relative to the nearest guarantees.}
\label{tab:nearest}
\fontsize{8.5}{10.3}\selectfont
\setlength{\tabcolsep}{4.5pt}
\begin{tabularx}{\linewidth}{@{}
  >{\raggedright\arraybackslash\hyphenpenalty=10000\relax}p{.302\linewidth}
  >{\raggedright\arraybackslash\hyphenpenalty=10000\relax}p{.277\linewidth}
  >{\raggedright\arraybackslash\hyphenpenalty=10000\relax}X@{}}
\toprule
\textbf{Work} & \textbf{Setting or result} & \textbf{What we study}\\
\midrule
\citet{burgess2021}&Sampling by group with at least two initial draws per group&Repeat a subset of tasks and find the best possible width rate as tasks and repeats grow\\
\addlinespace[2.5pt]
\citet{cai2004}&Valid intervals whose average width adapts to the data&Fixed binary outcomes under a strict total budget\\
\addlinespace[2.5pt]
\citet{waudby2020},\newline\citet{spertus2024}&Coverage as sampling continues, including sampling by group&Optimal worst-case width when runs agree within tasks, even with adaptive sampling\\
\addlinespace[2.5pt]
\citet{feldman2026}&Estimate event rates by deciding whether to continue runs&A strict budget and width bounds from repeated runs\\
\addlinespace[3pt]
\textbf{This paper}&Cover every task or allow some to be skipped&Optimal width rates and a condition for narrower intervals than uniform sampling\\
\bottomrule
\end{tabularx}
\end{table}
\endgroup

\Needspace{11\baselineskip}
\section{Problem setup: measuring a prespecified evaluation grid}\label{sec:problem}

The fixed cohort $\mathcal C$ contains $M$ tasks with $L$ planned runs per
task, for $N=ML$ runs. We call each run a \emph{path} of up to $K$ model
responses or complete agent executions, also called episodes. A
\emph{decisive event} is the outcome that ends the run, either the first
success or the first failure, depending on the score. Let $T_{ir}$ be the
step at which it occurs on path $r$ of task $i$, with $T_{ir}=K+1$ if it
does not occur. The path's final, or \emph{terminal}, label
$Y_{ir}=\ind\{T_{ir}>K\}$ equals one (a positive label) if no event occurs
and zero otherwise.
The full planned score is the mean of these labels,
\begin{equation}\label{eq:target}
 \theta_{\mathcal C}=\frac{1}{ML}\sum_{i=1}^M\sum_{r=1}^L Y_{ir}.
\end{equation}
With first success as the event, $1-\theta_{\mathcal C}$ is the fraction of
paths with at least one success. With first failure, $\theta_{\mathcal C}$
is the fraction with $K$ successes. The paths use preassigned, nonoverlapping
groups of attempts.

A sampling \emph{policy} selects observations using task identities and
past results. It can extend an already observed initial segment, called a
\emph{prefix}, from step $a$ to step $b$ if no decisive event has occurred.
This costs $\min(T_{ir}-a,b-a)$ new responses or episodes and reveals
whether the event occurred. The total cost $C$ must be at most the budget
$B$ for every fixed cohort and every set of random choices. Cost here
counts responses or episodes. Measuring tokens, tool calls, money, or time
would require a different unit.

\paragraph{What coverage means.}
Let $D$ be the observed data, $I(D)$ the resulting interval, $R$ the
evaluator's random choices, and $\alpha$ the allowed probability of missing
the score. Coverage requires
$\Prb_R\{\theta_{\mathcal C}\in I(D)\mid\mathcal C\}\ge1-\alpha$ for every
fixed cohort, even when outcomes differ or are dependent. We write $\E_R$
for an average over those choices and $\Var$ for variance, the average
squared deviation from the mean. Appendix~\ref{app:bridge} also treats the
average probability of no decisive event on new runs from the same tasks.
We call a second run on a randomly selected task an \emph{audit} because
it checks whether two outcomes on that task agree.

\section{How replication controls uncertainty}\label{sec:limits}

We first ask what completed and repeated runs reveal about the full
score. We then prove two replication laws for cohorts with agreement
within every task. For each design, we average the interval width over
random draws and take the worst case over such cohorts. We seek the
smallest such value across designs. The laws are \emph{sharp} because
an achievable upper bound and an unavoidable lower
bound depend on the task count and repeat budget in the same way, up to
constant factors. The lower bounds also cover adaptive sampling, which
uses earlier outcomes to choose later runs.

\subsection{Information from completed outcomes}

Early stopping can save substantial cost when decisive events occur early.
An honest guarantee must also cover cohorts whose decisive events occur only
at the end. There, incomplete paths reveal no terminal labels regardless of
later adaptive choices. The first theorem bounds the mean squared error
(MSE), the average squared difference between the estimate and the full
planned score over the evaluator's random choices.

\begin{theorem}[Late-event hard-budget limit]\label{thm:budget}
The budget can pay for at most $m=\min\{N,\lfloor B/K\rfloor\}$ paths that
each cost $K$, where $\lfloor\cdot\rfloor$ rounds down. The worst-case MSE
of any estimator is at least
\begin{equation}\label{eq:bayes}
 \frac{(N-m)(N+2)}{6N^2(m+2)}.
\end{equation}
If $B<K$, there are two indistinguishable cohorts with targets zero and one,
so the worst-case MSE is at least $1/4$.
\end{theorem}

To see the intuition, place every decisive event at step $K$ or later. Every
completed label then costs $K$, and an incomplete path provides no terminal
information. Giving each possible number of positive paths equal
probability yields Equation~\eqref{eq:bayes}. When $B/K$ grows but remains
well below $N$, this lower bound scales as $K/B$. Early events can make
more observations affordable on particular cohorts.

\subsection{Why repeated runs help quantify uncertainty}

One run per task gives an estimate that weights all tasks equally. An
interval may still need to be wide even when that estimate happens to
equal the full planned mean. We write $|I|$ for the width of interval $I$,
the upper endpoint minus the lower endpoint, and
$\Bin(M,q)$ for the count of successes in $M$ independent trials with
success probability $q$.

\begin{theorem}[One completed path per task]\label{thm:interval}
Suppose $L\ge2$ and write $q=\lfloor L/2\rfloor/L$. An evaluator completes exactly one uniformly sampled
path per task and observes no other terminal labels. Any additional observed
prefixes have depth below $K$. Let $I(D)$ be an interval with coverage at least
$1-\alpha$ for every fixed cohort. There exists a cohort in which every task's
$L$ terminal labels are identical, and one-per-task averaging is exact, yet
$\E_R|I(D)|\ge b_{M,q}(\alpha)$, where for $Z\sim\Bin(M,q)$,
\begin{equation}\label{eq:trim}
 b_{M,q}(\alpha)=\inf_{0\le w\le1,\,\E w(Z)\ge1-2\alpha}
 \E\bigl[|Z/M-q|w(Z)\bigr].
\end{equation}
Here $w(Z)$ is a weight between zero and one on each possible count.
For $0<\alpha<1/2$,
$\sqrt M\,b_{M,q}(\alpha)\to
2\sqrt{q(1-q)}\{1-\exp[-z_{1-\alpha}^2/2]\}/\sqrt{2\pi}$, where the standard
normal has probability $u$ below $z_u$. At $\alpha=.05$, the constant is $0.2958$ for even $L$ and
$0.2898$ for $L=5$.
\end{theorem}

The proof compares two kinds of cohorts that look identical after one draw per
task. In the first, all runs within a task agree. In the second, exactly
$\lfloor L/2\rfloor$ paths per task are positive. Their targets differ, so an
interval that is honest for both must remain wide enough to accommodate the
unseen alternatives. This is the paper's central gap between \emph{estimating}
the mean and \emph{certifying} the estimate. Replication closes the gap by
directly measuring agreement within tasks. Complete proofs appear in
Appendix~\ref{app:limits}.

\subsection{Partial replication while covering every task}

Let $t$ be the extra budget in units of a full $K$-step path, so
$B=(M+t)K$. Select one path uniformly per task and, in a uniform
subset $S$ of $t$ tasks, a second distinct path. Choose all paths
before observing labels. If $A_i$ averages the observed labels for
task $i$, the estimate $\widehat\theta=M^{-1}\sum_i A_i$ is
\emph{unbiased}, meaning its average over random draws equals the
target. This design leaves any savings from early stopping unused.

Let $\mathcal P$ contain the \emph{pure cohorts}, where all planned
runs within each task agree. Let $\Pi_{M,t}$ contain policies that
always observe at least one label per task and respect the budget
$(M+t)K$. They may change choices based on earlier results, stop
early, and spend the savings. For each policy $\pi$,
$\mathcal I_\alpha(\pi)$ contains its intervals that are honest over
\emph{all} fixed cohorts.

\begin{theorem}[Sharp width rate with full task coverage]\label{thm:partial}
For fixed $L\ge3$ and $0<\alpha\le1/12$, there are constants
$0<c_{\alpha,L}\le C_{\alpha,L}<\infty$, independent of $M,t$, such that
\begin{equation}\label{eq:partial_rate}
 \frac{c_{\alpha,L}}{\sqrt{M(t+1)}}\le
 \inf_{\pi\in\Pi_{M,t},\,I\in\mathcal I_\alpha(\pi)}
 \sup_{\mathcal C\in\mathcal P}\E_R|I|
 \le\frac{C_{\alpha,L}}{\sqrt{M(t+1)}},\qquad 0\le t\le M.
\end{equation}
Auditing a fixed-size random subset attains the upper bound on every pure
cohort. For each policy and honest interval, at least one pure cohort
has average width at or above the lower bound.
\end{theorem}

Repeating a fixed positive fraction $\rho$ of tasks, with
$t=\lfloor\rho M\rfloor$, improves the width from the one-run
$M^{-1/2}$ scale to $M^{-1}$ as $M$ grows. If $t$ grows
while $t/M$ tends to zero, the width scales as $(Mt)^{-1/2}$ and
the total budget relative to $MK$ approaches one.

\subsection{Trading limited task coverage for additional replication}

If the evaluator may omit a small random subset of tasks, the labels saved on
those tasks can fund more replication. The interval then combines uncertainty
from task sampling with uncertainty about unobserved runs within tasks.

Let $\Pi^{\rm all}_{M,t}$ contain every policy that respects the budget
$(M+t)K$, has no outside information about the current cohort, and may
skip tasks. The coverage guarantee and the pure cohorts $\mathcal P$
are as above.

\begin{theorem}[Sharp width rate when tasks may be skipped]\label{thm:omission}
For fixed $L\ge3$ and $0<\alpha\le1/12$, constants
$0<c'_{\alpha,L}\le C'_{\alpha,L}<\infty$ exist such that, uniformly for
$M\ge1$ and $0\le t\le M$,
\begin{equation}\label{eq:omission_rate}
 \frac{c'_{\alpha,L}}{\sqrt{M(t+\sqrt M)}}\le
 \inf_{\pi\in\Pi^{\rm all}_{M,t},\,I\in\mathcal I_\alpha(\pi)}
 \sup_{\mathcal C\in\mathcal P}\E_R|I|
 \le\frac{C'_{\alpha,L}}{\sqrt{M(t+\sqrt M)}}.
\end{equation}
\end{theorem}

For the upper bound, when $M\ge16$ and $t^2<M$, select
$n=M-\lfloor\sqrt M\rfloor$ tasks uniformly and use the saved task labels to
audit $q=t+\lfloor\sqrt M\rfloor$ of them; otherwise use the task-covering
design. We add a bound for the skipped tasks to the interval for the
sampled tasks. At $t=0$, the added uncertainty from skipping about $\sqrt M$
tasks and the uncertainty about unobserved runs are both of order
$M^{-3/4}$. Appendix~\ref{app:omission} gives the construction.

\paragraph{Why no policy can improve these rates.}
The proof places decisive events at the end of each run, so incomplete
runs reveal no labels. It compares cohorts whose outcomes agree within
each task with alternatives that hide a few opposite outcomes among the
planned runs. The full scores can differ even when the observations
match, so an adaptive evaluator can face the same choices in both cases.
An honest interval must account for both possible scores. Skipping only
a few tasks frees little extra budget for finding hidden differences.
Skipping many tasks creates uncertainty about their scores. The lower
bound accounts for both sources of uncertainty (Appendix~\ref{app:partial}).

\subsection{From optimal rates to a practical interval}

The \emph{Audit} interval uses the number $d$ of observed pairs that
disagree. For $t\ge1$ and allowed failure probability $\delta$,
$U(d,\delta;t)$ is an upper confidence bound on the expected value of
$d$. It is the largest $u\in[d,t]$ with
$u-d+d\log(d/u)\le\log(1/\delta)$, with a zero log term if $d=0$.
Use the variance bound $V_U$, single-task error bound $c_0$, and
interval half-width $r$ given by
\begin{equation}\label{eq:main_partial}
 V_U=\min\left\{\frac1{4M},\frac{(L-1)U(d,\alpha/2;t)}{2LMt}\right\},
 \quad c_0=\frac{L-1}{LM},\quad
 r=\frac{c_0x}{3}+\sqrt{2V_Ux+(c_0x/3)^2},
\end{equation}
where $x=\log(4/\alpha)$. The interval $\widehat\theta\pm r$ has
guaranteed coverage for the fixed-size random subset design. We intersect
it with scores that are still possible given the observed labels
(Section~\ref{sec:designs}). At $t=0$, we use Hoeffding's bound for averages
of independent bounded observations. On pure cohorts $d=0$, giving
width at most a constant times $(Mt)^{-1/2}+M^{-1}$ for $t\ge1$.
Appendix~\ref{app:partial} gives the proof. Coverage for audits chosen
using earlier outcomes requires a separate argument.

\paragraph{Intervals for a given budget.}
At $t=M$, every task has two draws, and $\Var(A_i)=(L-2)\E D_i/(4L)$,
where $D_i$ is one for a disagreement in task $i$ and zero otherwise.
The \emph{Pair} interval uses the three possible values of the task
mean and disagreement, $(0,0)$, $(1/2,1)$, and $(1,0)$. For $0<t<M$,
\emph{Joint} bounds the score error and disagreement count together,
averaging over the random choice of repeated tasks. Joint uses
Hoeffding at $t=0$ and Pair at $t=M$. Audit proves the optimal width
rate, and Joint is the interval used in the main experiment
(Appendix~\ref{app:joint_partial}).

\paragraph{When should Joint improve on uniform sampling?}
Write $h_i$ for the number of positive paths in task $i$ and
$H=\sum_i h_i$ for the cohort total. The average chance that two uniform
distinct paths from a task disagree is
$\bar q=M^{-1}\sum_i 2h_i(L-h_i)/\{L(L-1)\}$.
Let $x_t(d)/M$ be Joint's half-width before limiting it to feasible scores.
Let $W_U(H)$ be the average width for uniform sampling of $M+t$ paths
from the full grid. Uniform sampling without replacement gives a
\emph{hypergeometric} distribution for the number of positive labels.
The \emph{equal-tailed} baseline interval uses this distribution and
assigns half the allowed error to each end of the count distribution.

\begin{proposition}[A sufficient condition for narrower Joint intervals]\label{prop:finite_region}
For $0<t<M$, $x_t$ is nondecreasing and concave,
$\E_R|I_{\rm Joint}|\le2x_t(t\bar q)/M$, and the computable threshold
$q_\star(H)=\sup\{q:2x_t(tq)/M\le W_U(H)\}$ is sufficient:
$\bar q\le q_\star(H)$ implies no larger expected width than fixed uniform.
\end{proposition}

Appendix~\ref{app:finite_region} gives the proof and exact sums. We check
the condition after evaluation using $H$ and $\bar q$ from the full grid.
Using it to choose a design in advance would require outside information
on both quantities. The condition leaves other cases open.

\section{Implementing the evaluation framework}\label{sec:designs}

Table~\ref{tab:design_overview} summarizes the two budget allocations and
the baseline, called \emph{exact fixed uniform}. That baseline selects a
fixed number of paths uniformly from the full grid and uses the exact
distribution of the observed count. We also compare with methods that
spend savings from early stopping on additional runs.

\begin{table}[H]
\centering
\caption{The framework's allocation choices and the uniform-sampling baseline.}
\label{tab:design_overview}
\small
\begin{tabular}{p{0.25\linewidth}p{0.65\linewidth}}
\toprule
Design & Budget allocation and role \\
\midrule
Task-covering replication & Selects one run per task and repeats a random subset; uses the Audit or Joint interval. \\
Task sampling with replication & Samples a subset of tasks and funds additional repeats; includes uncertainty from skipped tasks. \\
Exact fixed uniform & Samples a fixed number of task--run pairs without replacement and computes an interval from the exact count distribution. \\
\bottomrule
\end{tabular}
\end{table}

\paragraph{Complete-path baselines.}
Exact fixed uniform completes $\min(N,\lfloor B/K\rfloor)$ distinct paths,
and its sample mean is unbiased. \emph{Uniform adaptive completion} uses
any savings from early stopping to draw more paths. It uses a
\emph{confidence sequence}, a series of intervals that cover the fixed
score together with the stated probability, even if sampling stops early
\citep{waudby2020}. \emph{Balanced adaptive completion} samples paths within
tasks and visits tasks in a random order each round. Its estimate at the
budget limit is a baseline for MSE.

\paragraph{Additional budget allocation methods.}
In \emph{progressive sampling}, initial responses guide which unfinished
paths to advance. The task-aware version makes these choices by task.
The extended appendix gives the sampling rules, budget guarantees,
and intervals for these designs.

\paragraph{Use all purchased label information.}
If $P$ distinct paths have known label one and $Z$ have known label zero,
the target certainly lies in $[P/N,1-Z/N]$. We intersect every interval with
this feasible range. The refinement costs no extra observations and preserves
coverage; it is applied uniformly to all methods. It also gives the
cost-stopped balanced heuristic a range containing every outcome still
logically possible after the observations, reported separately from statistical
confidence intervals.

\begingroup
\setlength{\parskip}{4pt}
\setlength{\textfloatsep}{10pt plus 2pt minus 2pt}
\setlength{\abovecaptionskip}{5pt}
\section{Evaluating precision at a fixed budget}\label{sec:experiments}

We first test the two evaluation methods at equal output counts on
LiveCodeBench. We then compare intervals on the same sampled records to
examine the role of disagreements, and use synthetic cohorts and longer
runs to study how the results change across settings.

\Needspace{5\baselineskip}
\subsection{Confidence intervals at the same budget}
\label{sec:lcb_results}

\paragraph{LiveCodeBench replay design.}
For each of 16 model configurations (model versions and generation
settings), we take the first five recorded
pass or fail outcomes for each of 880 tasks ($M=880,L=5,K=1$). The mean
of all 4,400 binary labels is the reference score. At a budget of $M+t$ outputs, the
task-covering design picks one random output per task and a distinct second
output from each of $t$ randomly chosen tasks. At 1,100 outputs this is 880
first draws and 220 repeats. The uniform baseline draws the same number of
task--output pairs uniformly from the full grid without replacement.
We specified both designs and the budget grid before inspecting
correctness (Appendix~\ref{app:lcb}). In each replay, methods use only
the selected labels to estimate the score and construct 95\% intervals.
The task-covering method uses Joint. The baseline uses the exact
distribution of the number of successes under uniform sampling. We repeat
the sampling 1,000 times per configuration and budget, recording the
estimation error, interval width, whether the interval contains the
reference score, and the number of outputs used.

\Needspace{5\baselineskip}
\paragraph{More precision from the same 1,100 outputs.}
For each configuration, we average each method's interval widths over its
replays and divide Joint's average by the uniform baseline's average. At
1,100 outputs, this ratio is below one in 15 of 16 configurations. The median
of the 16 ratios is 0.694, a $1-0.694=30.6\%$ reduction. The median ratio of
point-estimate MSE is 0.130, an 87.0\% reduction, and MSE is lower in every
configuration at every tested budget. At $t=15$ and $t=88$, the median width
ratios are 0.966 and 0.807, with narrower intervals in 9/16 and 14/16
configurations. At $t=440$ and $t=880$, the ratios are 0.673 and 0.677, with
narrower intervals in 14/16 configurations at both budgets.

\begin{figure}[t]
\centering
\begin{tikzpicture}[font=\scriptsize]
  \node[anchor=west,font=\bfseries] at (0,3.15) {(a) Optimal width rates};
  \node[draw=blue!65!black,rounded corners=2pt,fill=blue!4,
        align=left,text width=5.3cm,inner sep=5pt,anchor=north west] at (0,2.9) {
    \textbf{Every task observed}\hfill repeat random tasks\\[-1pt]
    $\displaystyle \Theta_{\alpha,L}\!\left([M(t+1)]^{-1/2}\right)$\\[-1pt]
    $t=0:\ M^{-1/2}$ \qquad $t=\rho M:\ M^{-1}$
  };
  \node[draw=teal!70!black,rounded corners=2pt,fill=teal!4,
        align=left,text width=5.3cm,inner sep=5pt,anchor=north west] at (0,1.45) {
    \textbf{Task omission allowed}\hfill skip random tasks\\[-1pt]
    $\displaystyle \Theta_{\alpha,L}\!\left([M(t+\sqrt M)]^{-1/2}\right)$\\[-1pt]
    $t=0:\ M^{-3/4}$ \qquad $t=\rho M:\ M^{-1}$
  };

  \begin{scope}[xshift=7.35cm,x=0.55cm,y=1.55cm]
    \node[anchor=center,font=\bfseries] at (4.0,2.08) {(b) LiveCodeBench at equal label budgets};
    \draw[->] (0,0) -- (8.45,0) node[right,align=left] {$t$\\[-2pt]{\tiny extra runs}};
    \draw[->] (0,0) -- (0,1.75);
    \node[rotate=90,anchor=south] at (-0.92,0.88) {Joint / uniform mean width};
    \foreach \y/\lab in {0.5/{0.5},1/{1.0},1.5/{1.5}} {
      \draw (-0.08,\y) -- (0.08,\y);
      \node[left] at (-0.08,\y) {\lab};
    }
    \draw[dashed,gray] (0,1) -- (8.2,1);
    \node[gray,anchor=south east] at (8.2,1) {equal widths};
    \foreach \x/\lab in {0/0,2.59/8,3.27/15,4.03/{29/30},5.30/88,6.37/220,7.19/440,8/880} {
      \draw (\x,-0.05) -- (\x,0.05);
      \node[below,rotate=45,anchor=north east] at (\x,-0.04) {\lab};
    }
    \draw[very thick,blue!70!black]
      plot[mark=*,mark size=1.5pt] coordinates {
        (0,1.562) (2.59,1.105) (3.27,0.966) (4.01,0.907) (4.05,0.900)
        (5.30,0.807) (6.37,0.694) (7.19,0.673) (8,0.677)
      };
    \node[blue!70!black,fill=white,inner sep=1pt] at (3.05,1.13) {9/16};
    \node[blue!70!black,fill=white,inner sep=1pt] at (5.35,0.96) {14/16};
    \node[blue!70!black,fill=white,inner sep=1pt] at (6.62,0.54) {15/16};
    \node[anchor=west,align=left,text=blue!70!black] at (0.25,0.20)
      {Configurations with narrower intervals};
  \end{scope}
\end{tikzpicture}
\caption{Optimal width rates and equal-budget LiveCodeBench results.
Panel (a) gives the optimal worst-case average width when all runs agree
within each task. $\Theta_{\alpha,L}$ means matching upper and lower bounds
up to constants depending on error level $\alpha$ and runs per task $L$.
Lower bounds allow adaptive sampling. Here $M$ is the task count, $t$ the
extra-run budget, and $\rho$ the fraction of tasks repeated in the example.
Panel (b) shows the median across 16 configurations of the ratio of Joint to
uniform mean widths over replays at $M+t$ labels ($M=880,L=5$). The axis uses
$\log_2(t+1)$ with ticks labeled by $t$. At $t=220$ (1,100 labels), task-covering
Joint is narrower in 15 of 16 configurations, with a median reduction of 30.6\%.}
\label{fig:main_results}
\end{figure}

Appendix~\ref{app:lcb} reports the budget grid, coverage checks, and
per-configuration results at 1,100 outputs.

\subsection{How disagreements improve precision}
\label{sec:finite_evidence}

\paragraph{Information in disagreement.}
After inspecting the full results, we compare Joint with \emph{Hull}, an
interval built from the same score estimate that leaves out the disagreement
count. Joint is narrower on average in all 16 configurations at every
tested $t\ge8$. The median Joint/Hull width ratio is 0.732 at $t=8$ and 0.425 at
$t=220$. On the nine-budget grid, all configurations reach mean width 0.04
with median requirements of 1,100 labels for Joint and 1,760 for uniform,
a median saving of 660 labels. Appendix~\ref{app:accept_gap} describes how
these budget requirements are computed.

\paragraph{Conditions for narrower intervals.}
Proposition~\ref{prop:finite_region} uses the share of positive labels
and the chance that two runs on the same task disagree to identify when
Joint is at least as narrow on average. On the complete LiveCodeBench records,
the condition holds in 73 of 112 configuration--budget combinations with
$0<t<M$. Joint has narrower observed intervals in all 73 and in five
more combinations.

\paragraph{Controlled departures from within-task agreement.}
On constructed outcome grids with score $1/2$ and $L=5$, we vary
the task count $M$ from 128 to 32,768, the fraction of tasks repeated,
and seven fractions of tasks with mixed outcomes (up to 20\%). Using exact
sampling probabilities evaluated numerically, Audit beats fixed uniform
on width in 185 of 236 settings. In simulations with every task repeated,
Pair beats it in 57 of 59. With about 10\% of tasks repeated and no
mixed outcomes, the width ratio falls from 1.307 at $M=128$ to 0.661
at $M=512$ (Appendix~\ref{app:synthetic}).

\Needspace{9\baselineskip}
\subsection{Evaluation across budgets and settings}

\paragraph{Longer paths at equal hard budgets.}
For paths of up to 1,000 responses ($K=1000$), we replay both halves of all
22 grids of recorded outcomes at six budgets, stopping at either the first
success or the first failure, in separate experiments. Joint is narrower
than exact fixed uniform in 27/528 conditions, with a median width ratio
of 1.509, and has lower MSE in 300/528. Twenty width wins come from
first success in two configurations on the MiniF2F-MATH math benchmark.
For agent paths of up to two episodes ($K=2$), Joint is narrower in
6/144 conditions, Audit in 0/144, and Pair in 1/24 when every task is repeated.
Both comparisons use previously inspected records at the same maximum cost
(Appendix~\ref{app:kgt1_response}).

\paragraph{Task coverage and finite-budget comparisons.}
At the four LiveCodeBench budgets where the design skips tasks
($t=0,8,15,29$), its median width ratios to exact fixed uniform are
1.346, 1.328, 1.324, and 1.287. The ratios exceed one in all 16
configurations at each budget (Appendix~\ref{app:lcb}).
Holding sampling fixed in 19 small cases, Joint's worst-case average
width on pure cohorts is 1.116--1.818 times the width obtained by
numerically optimizing over valid interval rules (Appendix~\ref{app:accept_gap}).
Across 868 historical conditions and 2,492,000 replays, task-aware
progressive sampling has median MSE ratios 1.02 on responses and 1.30 on
agents against balanced adaptive completion, and width ratios 1.66 and 1.90
against uniform adaptive completion. The extended appendix gives the complete
historical comparisons.

Table~\ref{tab:public_replay_evidence} summarizes the full Joint grids and
the agent Audit and Pair comparisons against exact fixed uniform.

\begin{table}[!tb]
\centering\small
\caption{Counts of conditions with narrower, wider, or equal mean interval
width versus exact fixed uniform at equal reserved budgets. $K$ caps responses
or agent episodes per path. The Pair row repeats every task.}
\label{tab:public_replay_evidence}
\begin{tabular}{llrlrrr}
\toprule
Benchmark & Interval & $K$ & Analysis specified & Wins & Losses & Ties \\
\midrule
LiveCodeBench & Joint & 1 & Before seeing outcomes & 92 & 52 & 0 \\
Agent episodes & Joint & 2 & After seeing outcomes & 6 & 138 & 0 \\
Agent episodes & Audit & 2 & After seeing outcomes & 0 & 144 & 0 \\
Agent episodes & Pair & 2 & After seeing outcomes & 1 & 23 & 0 \\
Recorded responses & Joint & 1000 & After seeing outcomes & 27 & 501 & 0 \\
\bottomrule
\end{tabular}
\end{table}

\Needspace{16\baselineskip}
\section{Conclusion}

Under a fixed evaluation budget, repeated runs help quantify how far
the estimated score may be from the full benchmark score. Our
framework allocates runs based on whether every task must be covered and
uses disagreements within tasks to construct a confidence interval.

Intervals built with Audit have guaranteed coverage for every fixed set
of outcomes, with an allowance for skipped tasks when needed. When all
planned runs agree within each task, the designs achieve the optimal
worst-case rate for average interval width. This holds whether or not
every task must be covered. The matching lower bounds apply even to
evaluators that choose later runs based on earlier results.

In the LiveCodeBench replay, the full-coverage design using Joint and the
uniform baseline each observe 1,100 outputs per model configuration and
use nominal 95\% intervals. Joint has smaller average width in 15 of 16
configurations, with a median reduction of 30.6\%. On the same sampled
outcomes at this budget, Joint also gives narrower intervals in all 16
configurations than Hull, which leaves out the disagreement count.

\clearpage
\endgroup
\subsection*{AI use statement}
The core theoretical ideas, mathematical claims, derivations, proofs, and
experimental design were developed by the author. Generative artificial
intelligence (AI) tools were
used to support research exploration and literature review, assist with data
organization, plotting, utility scripts, and general code editing, and polish
the language of the manuscript. The author has completed a manual review of
the manuscript, proofs, code, results, and all AI-assisted material, and takes
full responsibility for the final content.

\subsection*{Ethics statement}
This work uses public research outcome records and does not recruit human
participants or deploy an agent into a live external environment. The replay
code consumes correctness and task identifiers, not private communications.
The response-bank and agent-artifact licenses and primary sources are retained
in the reproducibility record. Efficiency estimates must not be interpreted
as safety guarantees or as evidence that unobserved failures did not occur.

\subsection*{Reproducibility statement}
The reproducibility records include the LiveCodeBench plan set before
examining correctness and the response study planned after inspecting
outcomes. They include exact source versions, file fingerprints for checking
downloads and outputs, code, random seeds, grading checks, and replay results.
The appendix gives the assumptions for each interval. The code reveals only
the initial segment observed so far and counts each response or episode
actually observed. The replays use already recorded model outputs.

\bibliography{references}
\bibliographystyle{iclr2027_conference}
\clearpage
\appendix
\raggedbottom
\section*{Supplementary Material (Appendix)}
\noindent The extended appendix in the manuscript source contains the full
derivations and replay results.

\section{Information limits and the population target}\label{app:limits}

\paragraph{Proof of Theorem~\ref{thm:budget}.}
Set every decisive event at depth $K$. Then shorter prefixes are uninformative
and each completed path costs $K$. Use a prior, a probability distribution
over fixed cohorts, that draws the positive count $H$ uniformly and places
the $H$ positives uniformly. The unqueried positions are symmetric, so
adaptive queries have the same label distribution as uniform draws without
replacement. Allowing $m$ draws can only lower Bayes risk, the smallest
average squared error under the prior. Equivalently, draw $p$ uniformly on
$[0,1]$ and then generate all $N$ labels independently with probability $p$
of being positive. Given $s$ positives in $m$ draws, the density of $p$
is proportional to $p^s(1-p)^{m-s}$ and the unseen positive count is
binomial with $N-m$ draws and success probability $p$ conditional on $p$.
Thus $\Var(H\mid s)=(N-m)(s+1)(m-s+1)(N+2)/[(m+2)^2(m+3)]$.
Averaging over uniform $s\in\{0,\ldots,m\}$ and dividing by $N^2$ gives
the Bayes risk in Equation~\eqref{eq:bayes}. If $B<K$, all-zero and all-one
cohorts give identical observations, and the larger of their mean squared errors is at
least $1/4$.

\paragraph{Proof of Theorem~\ref{thm:interval}.}
Use the same late-event restriction. Under one prior, each task is constant
and its label independently equals one with probability $q$. Under the
other, each task has exactly $qL$ positives uniformly placed. The record of
queries and observed labels, called a transcript, has
the same law under both priors. Honesty implies that the interval contains
both the sampled mean $Z/M$ and $q$ with probability at least $1-2\alpha$.
Conditioning that event on $Z\sim\Bin(M,q)$ yields
Equation~\eqref{eq:trim}. Some fixed pure cohort has at least this
prior-average expected width. To minimize the weighted average, keep the
smallest values of $|Z/M-q|$ until their probability totals $1-2\alpha$,
allowing partial weight at the cutoff. The central limit theorem gives the
normal limit and the bounded second moments control its tails. The result
follows by integrating the absolute standard normal over
$[-z_{1-\alpha},z_{1-\alpha}]$.

\subsection{Population target}\label{app:bridge}
If all $ML$ planned paths are independent and the paths within task $i$
have common probability $p_i(K)$ of no decisive event by step $K$, then
$\E\theta_{\mathcal C}=\theta_* = M^{-1}\sum_i p_i(K)$. For an allowed
probability $\beta$ of error from generating the cohort, Hoeffding bounds
$|\theta_{\mathcal C}-\theta_*|$ by $\sqrt{\log(2/\beta)/(2ML)}$ except
with probability $\beta$. Expanding both interval endpoints by this
amount covers $\theta_*$ with probability at least $1-\alpha-\beta$.

\section{Finite-budget intervals}\label{app:joint_partial}

\subsection{Full audits}
Writing $p_i=h_i/L$ for the positive fraction of task $i$, the variance
of the two-draw mean $A_i$ is $(L-2)\E D_i/(4L)$.
Moreover, $|A_i-p_i|\le(L-2)/L$ by decomposition into sampled and remaining
labels. The disagreement and variance bounds in Appendix~\ref{app:partial}
therefore apply with
$V_U=(L-2)U(d,\alpha/2;M)/(4LM^2)$ and $c_0=(L-2)/(LM)$
in Equation~\eqref{eq:main_partial}.

\subsection{Joint and Pair}
Fix $0<t\le M$, $\rho=t/M$, and $p=h/L$ for a task with $h$ positives.
Let $q_k={\binom hk\binom{L-h}{2-k}}/{\binom L2}$ be the probability of
$k$ positives in two draws without replacement. For $u>0$, let $b_h$ be
the average of $e^{u(A-p)}$ when $A$ is one sampled label. When $A$ averages
two draws, $a_{0,h}$ and $a_{1,h}$ are the contributions to that average
from agreeing and disagreeing pairs, respectively:
\[
\begin{aligned}
 b_h&=(1-p)e^{-up}+pe^{u(1-p)},\\
 a_{0,h}&=q_0e^{-up}+q_2e^{u(1-p)},\\
 a_{1,h}&=q_1e^{u(1/2-p)}.
\end{aligned}
\]
For $t<M$ require $C_h=b_h(1-\rho^{-1}\log b_h)-a_{0,h}>0$
for all $0<h<L$, and set $\psi(u)=\max\{0,\max_h\log(a_{1,h}/C_h)\}$.
This condition holds for small enough $u>0$ since $C_h(0)=q_1>0$. The
penalty $\psi(u)$ discounts each disagreement in the exponential bound.
Set
$a_h=a_{0,h}+a_{1,h}e^{-\psi(u)}$, with $a_h=b_h=1$ on pure tasks.
Conditional on the audited subset, draws are independent. Averaging over
the subset and applying Maclaurin's inequality to $a_{h_i}/b_{h_i}$, which
bounds the average of size-$t$ products by the $t$th power of their factors'
mean,
gives
\[
 \E_R e^{uM(\widehat\theta-\theta_{\mathcal C})-\psi(u)d}
 \le e^{B_0}\Big(\frac1M\sum_i\frac{a_{h_i}}{b_{h_i}}\Big)^t
 \le e^{B_0}(1-B_0/t)^t\le1,\qquad B_0=\sum_i\log b_{h_i}.
\]
Here $0\le B_0<t$ since exponentiating the mean of the centered error
gives one, which is at most its average exponential by Jensen's inequality,
and $C_h>0$ implies $\log b_h<\rho$.
For Pair at $t=M$, use $C_h=1-a_{0,h}>0$ in the same penalty. Then each
$a_h\le1$, so the joint bound follows by independence. Swapping zero
and one preserves disagreements and reverses the sign of the score error.

For values $u_j$ satisfying the relevant condition above, chosen before
seeing labels, and positive weights $w_j$ with
$w_*+\sum_j w_j=1$ and $w_*\ge0$, the radius is $x(d)/M$, where
\begin{equation}\label{eq:joint_inverse}
 w_*+\sum_j w_j e^{u_jx(d)-\psi(u_j)d}=2/\alpha .
\end{equation}
The left side increases strictly from at most one at $x=0$ to infinity.
At the true score, the weighted sum for each error sign has expectation
at most one. By Markov's inequality, it exceeds $2/\alpha$ with probability
at most $\alpha/2$. Together the two signs give coverage at least $1-\alpha$.
Intersect with the feasible label range from Section~\ref{sec:designs}.

The implementation uses Hoeffding at $t=0$. For $0<t<M$, it takes the
largest $u_0$ in the grid $\sqrt\rho\,2^{k/32}$ for
$k=-256,\ldots,192$ with all $C_h>10^{-10}$, or uses Audit when none
qualifies. For Pair, $u_0$ is $M/(M+1)$ times the first positive root
of $a_{0,h}(u)=1$ over $0<h<L$. Both use $u_j=u_0/2^j$,
$w_j=1/[(j+1)(j+2)]$ for $0\le j<J=\lceil\log_2M\rceil+1$, and
$w_*=1/(J+1)$. For $0<t<M$, move the weight of any discarded $u_j$ to
$w_*$. All choices depend only on the design.

\subsection{Proof of Proposition~\ref{prop:finite_region}}\label{app:finite_region}
Use Equation~\eqref{eq:joint_inverse} to define $x_t(d)$ also for
fractional $d$. The log of its left side is convex in $(x,d)$.
The set where this log is at most $\log(2/\alpha)$ is
$x\le x_t(d)$, so $x_t$
is concave. Nonnegative penalties make $x_t$ nondecreasing. With
$\E_Rd=t\bar q$, Jensen's inequality and feasible-label clipping give
$\E_R|I_{\rm Joint}|\le2x_t(t\bar q)/M$. For $n=M+t$ pooled labels,
let $[a_s,b_s]$ be the positive-count interval after observing $s$
positives. The comparison uses
$W_U(H)=\E[(b_S-a_S)/N]$, where $S$ counts positives in $n$ uniform draws
without replacement from $N$ labels with $H$ positives.
The sufficient condition follows by comparing these two widths.

\subsection{Additional replay designs}
Progressive designs sample from paths still unresolved before advancing
them. The weights divide by the fractions of paths and tasks kept at each
step. The reservation rule leaves enough budget to finish a retained path.
The full-prefix design advances every path to the same initial depth or
until resolved, then samples from those still unresolved and weights by
their observed task counts. The extended appendix gives the precise
schedules, reservations, and confidence procedures used in these
historical comparisons.

\section{Sharp rates with full task coverage}\label{app:partial}

\noindent\textbf{Audit upper bound.}
Let $p_i=L^{-1}\sum_r Y_{ir}$ be the full task mean. Imagine drawing two
ordered distinct labels $(X_i,X_i')$ uniformly and independently across
tasks. Reveal the second only for tasks in the independent uniform
$t$-subset $S$. Conditional on $S$, the sampled means $A_i$ are independent
with means $p_i$. For $t\ge1$, the disagreement probability for task $i$ is
$q_i=2Lp_i(1-p_i)/(L-1)$, with average $\bar q=M^{-1}\sum_i q_i$. The
disagreement count $d=\sum_{i\in S}\ind\{X_i\ne X_i'\}$ satisfies, by Maclaurin's
inequality, for $\lambda\ge0$,
\[
 \E e^{-\lambda d}
 =\binom Mt^{-1}\sum_{|S|=t}\prod_{i\in S}(1-q_i+q_i e^{-\lambda})
 \le (1-\bar q+\bar q e^{-\lambda})^t
 \le \exp\{t\bar q(e^{-\lambda}-1)\}.
\]
For $\mu=t\bar q$ and $k<\mu$, apply Markov's inequality to
$e^{-\lambda d}$ and optimize over $\lambda$ to bound $\Prb(d\le k)$ by
$\exp\{-\mu+k-k\log(k/\mu)\}$, including $k=0$ by continuity.
The event $\mu>U(d,\delta;t)$ consists of counts up to a cutoff, so
the bound at its largest count gives probability at most $\delta$.
Conditional on $S$, two labels drawn without replacement have nonpositive
covariance, hence
$\Var(\widehat\theta\mid S)\le V_0=(L-1)\bar q/(2LM)\le1/(4M)$.
Each centered summand $(A_i-p_i)/M$ has absolute value at most
$c_0=(L-1)/(LM)$ by the finite label grid. Bernstein's inequality bounds
the chance that a sum of independent bounded variables differs from
its mean, using their variances. With $x=\log(4/\alpha)$ it bounds the
two-sided error at the radius in
Equation~\eqref{eq:main_partial} with $V_0$ in place of $V_U$, uniformly
in $S$, with failure probability at most $\alpha/2$. Combining it with
the event $V_0\le V_U$ proves honesty. On pure cohorts $d=0$, and
$U(0,\alpha/2;t)\le\log(2/\alpha)$ gives width
$O_{\alpha,L}((Mt)^{-1/2}+M^{-1})$. Here $O_{\alpha,L}(g)$ means at most
a constant depending on $\alpha,L$ times $g$. At $t=0$, Hoeffding gives
$O_\alpha(M^{-1/2})$. Since $t\le M$, these bounds give the upper rate
in Theorem~\ref{thm:partial}.

\noindent\textbf{Lower-bound setup for both policy classes.}
Set $T_{ir}=K$ for label zero and $K+1$ for label one. Every revealed
terminal label then costs $K$ and shorter prefixes reveal no label information.
If $r_i$ distinct terminal labels are revealed in task $i$ and
$z=\#\{i:r_i=0\}$, then
$\sum_i r_i\le M+t$ and $\sum_i(r_i-1)_+\le t+z$, where
$(r_i-1)_+=\max\{r_i-1,0\}$ counts repeats in task $i$. Under prior A, each
task is pure and its label is independently zero or one with equal
probability. Write $\Prb_A,\E_A$ for probability and expectation under
this prior and the evaluator's random choices. For task-covering policies set
$s_0=0$, so $z=0$. For omission-enabled policies with $M\ge256$, set
$s_0=\lceil\sqrt M\rceil$. Write $F=\{z\le s_0\}$. Conditional on the
full transcript, including policy and interval random seeds, the $z$
unobserved labels remain independent and fair. On $F^c=\{z>s_0\}$ an interval of
width less than $\ell=\sqrt{s_0}/(4M)$ contains at most $\sqrt z/2$
possible values of their contribution to the target. Each value of
the binomial count $\operatorname{Bin}(z,1/2)$ has probability at most
$1/\sqrt z$. Averaging coverage over the prior thus gives
$\Prb_A(F^c,|I|<\ell)\le2\alpha$. If
$\Prb_A(F^c)\ge1/4$, the expected width is at least
$\ell(1/4-2\alpha)\ge\sqrt{s_0}/(48M)$, sufficient for
Theorem~\ref{thm:omission}.

\noindent\textbf{Compatible transcripts and target separation.}
Otherwise $\Prb_A(F)>3/4$. Set $D=t+s_0+1$ and $\epsilon=1/(4D)$.
Under prior B, independently replace each pure task with probability
$\epsilon$ by a fair majority bit with one uniformly placed opposite
bit. For a transcript $T$ consistent with within-task agreement, its
likelihood under B is $w(T)=\prod_i f_{r_i}$ times its likelihood under A,
where $f_0=f_1=1$ and $f_r=1-\epsilon r/L$ for $r\ge2$. On $F$,
\[
 1\ge w(T)\ge 1-2\epsilon(t+z)/L\ge 5/6.
\]
Indeed, for $r\ge2$ a task with an opposite label can match only if the
evaluator does not query that label. Given the same past and seeds, the
evaluator makes the same choices. The bound uses $r\le2(r-1)$ for $r\ge2$.
Weight the A transcript law by $\ind_F w$, then draw the two full
cohorts independently from their conditional laws given that transcript.
This gives a joint measure $Q$ on the transcript and the two cohorts with
total weight between $5/8$ and one. An event involving only one cohort and
the transcript has weight no greater than its probability under that
cohort's prior. If $\theta_A,\theta_B$ are the two full scores, then
$Q(\theta_A\notin I)+Q(\theta_B\notin I)\le2\alpha$ and
$\E_A|I|\ge\E_Q|I|$.

Task labels remain independent across tasks within each prior given the
transcript. On $F$ at
least $M-z-(t+z)/(L-1)\ge M/4$ observed tasks are not fully sampled.
For each such task, the variance of its mean under B given the transcript
is at least $\epsilon/(2L^3)$. For $r=1$, the task mean equals the observed
bit with probability $1-\epsilon$ and is shifted by $1/L$ toward the opposite
bit with probability $\epsilon(L-1)/L$. These two values alone give the bound.
For $2\le r<L$, the shift by $1/L$ has probability
$\epsilon(L-r)/(Lf_r)\in[\epsilon/L,1/3]$ and the other possible mean is
the observed bit. Write $H=M(\theta_B-\theta_A)$ as a sum of B task
means minus A task means, where observed A means are fixed. The remaining
contributions are independent with centered increments bounded by one
and conditional variance $v=\Var(H\mid T)\ge a_L M/D$, where $a_L=1/(32L^3)$.

Write $\mu=\E_Q(H\mid T)$ and $S_2=\mu^2+v$. Independence and bounded
increments give $\E_Q(H^4\mid T)\le3S_2^2+4|\mu|v+v$.
For $v\ge16$, dividing the last two terms by $S_2^2$ gives at most
$9/(4\sqrt{3v})+1/v<1/2$. Thus if $a_LM/D\ge16$,
the Paley--Zygmund inequality, a lower bound on the probability of exceeding
a fraction of the mean for a nonnegative variable, applied to $H^2$ gives
$Q(|H|\ge\sqrt{S_2/100}\mid T)\ge9801/35000$.
The interval must span both targets when both are covered. Excluding the
cases in which either score is outside the interval yields
\[
 \E_A|I|\ge
 \left(\tfrac58\tfrac{9801}{35000}-2\alpha\right)
 \frac{\sqrt{a_L/100}}{\sqrt{MD}}.
\]
The coefficient is positive for $\alpha\le1/12$.

\noindent\textbf{Small-variance and small-$M$ cases.}
Compare the all-zero cohort with a prior placing one positive label
uniformly among the $ML$ positions. The probability under this prior of
obtaining an all-zero transcript and missing the positive is its
probability under the all-zero cohort times
$1-\sum_i r_i/(ML)\ge1/3$. Honesty under the two laws then implies that
under the all-zero law the interval contains both $0$ and $1/(ML)$
with probability at least $1-4\alpha$. The resulting width bound
$(1-4\alpha)/(ML)$ is of the required order if $a_LM/D<16$ or
$M<256$. For full coverage $D=t+1$. With omissions, $D$ is within
constant factors of $t+\sqrt M$. Since prior A is supported on pure cohorts, the
average lower bounds also bound the worst pure cohort, proving the
lower halves of Theorems~\ref{thm:partial} and~\ref{thm:omission}.

\Needspace{14\baselineskip}
\section{Sharp rates with task omission}\label{app:omission}

The lower bound is proved in Appendix~\ref{app:partial}. For the upper
bound, omit a uniform $s$-subset, then select $q=t+s$ of the remaining
$n=M-s$ tasks uniformly for a second run, with $q\le n$. The reservation
is $(n+q)K=(M+t)K$. Apply Audit at error level $\alpha/2$ to the $n$
retained tasks. If $\mu_H$ and $\mu_O$ are the averages of $p_i$ over the
retained and omitted tasks, then
$\mu_H-\theta_{\mathcal C}=(s/n)(\theta_{\mathcal C}-\mu_O)$.
Hoeffding's bound for sampling without replacement
\citep{hoeffding1963,bardenet2015} gives the
additional radius $\min\{s/M,\sqrt{s\log(4/\alpha)/2}/n\}$ at failure
probability at most $\alpha/2$, where $s/M$ is deterministic. Adding the
two radii gives honesty since the probability that either bound fails is
at most the sum of their failure probabilities. On pure cohorts the width is
$O_{\alpha,L}((n(q+1))^{-1/2}+n^{-1}+\sqrt s/n)$.
For $M\ge16$ and $t^2<M$, use $s=\lfloor\sqrt M\rfloor$, so $q\le n$,
$n\ge M/2$ and the width is $O_{\alpha,L}(M^{-3/4})$. For $t^2\ge M$
use full coverage at error level $\alpha$, and for $M<16$ the unit
interval suffices. Together these give the upper rate in
Theorem~\ref{thm:omission}.

\section{LiveCodeBench evaluation}\label{app:lcb}

\paragraph{Protocol.}
Using the official code-generation samples of \citet{jain2025livecodebench}
(revision \path{79837278b7c58c64a17c90936207db25f977f414}), we require one
record for each of the 880 tasks and at least five binary correctness labels
per record. We take the first five labels in source order. Six single-output
configurations fail this rule, leaving 16 with $M=880$, $L=5$, and $K=1$.
A local protocol fixed the input files by their recorded hashes
(digital fingerprints), the inclusion rule, the designs, the budgets, and
the random seeds before correctness inspection. It emphasized task
omission and also specified task-covering Joint and exact uniform.

Each design uses $M+t$ labels and 1000 randomizations per configuration and
budget. The independent sequences of random draws are generated from
seed 20261010. Joint samples
one uniform label per task and a distinct second label in a uniform $t$-task
subset. Exact uniform draws from all $ML$ labels without
replacement and uses its exact equal-tailed interval. The design that
allows skipped tasks, \emph{Omission-Joint}, skips 29 uniformly chosen
tasks and repeats $t+29$ uniformly chosen retained tasks when $t^2<M$.
It divides the error probability equally between Joint on the retained
tasks and a bound for the skipped tasks, then intersects the resulting
interval with the feasible label range. At other budgets it uses Joint.

\begin{table}[H]
\centering\small
\setlength{\tabcolsep}{3.5pt}
\caption{LiveCodeBench budget grid. Ratios are medians across 16 configurations.
The win count compares Joint with exact uniform on mean interval width.
Maximum observed Joint coverage is 1 at every budget.}
\label{tab:lcb}
\begin{tabular}{@{}lrrrrrrrrr@{}}
\toprule
Repeated tasks $t$ & 0 & 8 & 15 & 29 & 30 & 88 & 220 & 440 & 880 \\
\midrule
Joint/uniform width & 1.562 & 1.105 & 0.966 & 0.907 & 0.900 & 0.807 & 0.694 & 0.673 & 0.677 \\
Joint/uniform MSE & 0.113 & 0.112 & 0.115 & 0.109 & 0.112 & 0.120 & 0.130 & 0.126 & 0.110 \\
Joint/Hull width & 1.040 & 0.732 & 0.639 & 0.595 & 0.591 & 0.516 & 0.425 & 0.406 & 0.449 \\
Width wins & 0 & 3 & 9 & 11 & 12 & 14 & 15 & 14 & 14 \\
Minimum Joint coverage & 1.000 & 0.999 & 0.995 & 0.995 & 0.995 & 0.992 & 0.998 & 0.998 & 0.999 \\
\bottomrule
\end{tabular}
\end{table}

\begin{table}[H]
\centering\small
\setlength{\tabcolsep}{3pt}
\caption{Joint relative to exact uniform at 1100 outputs for every eligible
configuration. Width is the ratio of mean interval widths.}
\label{tab:lcb_panels}
\begin{tabular}{@{}lrr@{\hspace{1.2em}}lrr@{}}
\toprule
Configuration & Width & MSE & Configuration & Width & MSE \\
\midrule
Claude-3-Haiku &0.704&0.122& GPT-4o-2024-05-13 &0.660&0.109\\
Claude-3.5-Sonnet-20240620 &0.559&0.060& GPT-4o-2024-08-06 &0.798&0.190\\
Claude-3.5-Sonnet-20241022 &0.662&0.119& GPT-4o-mini-2024-07-18 &0.667&0.121\\
Codestral-Latest &0.684&0.138& Gemini-Flash-1.5-002 &0.532&0.054\\
DeepSeek-R1-Lite-Preview &1.049&0.315& Gemini-Flash-2.0-Exp &0.567&0.072\\
DeepSeek-R1-Preview &0.964&0.270& Gemini-Flash-2.0-Thinking &0.839&0.221\\
DeepSeek-V3 &0.741&0.159& Gemini-Pro-1.5-002 &0.626&0.100\\
GPT-4-Turbo-2024-04-09 &0.745&0.177& Mistral-Large &0.741&0.155\\
\bottomrule
\end{tabular}
\end{table}

\Needspace{9\baselineskip}
\paragraph{Checks and interpretation.}
At nominal 95\%, observed coverage across the grid ranges from $0.992$ to
$1$ for Joint, $0.994$ to $1$ for Omission-Joint, and $0.938$ to $0.972$ for exact
uniform. All replays meet the budget. At $t=0,8,15,29$, Omission-Joint gives
median width ratios of $1.346$, $1.328$, $1.324$, and $1.287$ versus uniform,
with ratios above one in every configuration. We quantify uncertainty from
having finitely many replays using intervals based on the binomial count
of replays that cover the true score. For width differences we use
bootstrap intervals, obtained by resampling the replay results. The
fraction of tasks with identical outcomes across all
five runs has median $0.893$ and range $0.718$--$0.959$ across configurations.

\subsection{Controlled cohorts and retrospective diagnostics}
\label{app:synthetic}\label{app:accept_gap}
We compute expected widths from the exact sampling distributions for
$L=5$, full-cohort mean $1/2$,
$M\in\{128,512,2048,8192,32768\}$, and $t=\lceil\rho M\rceil$ for
$\rho\in\{.1,.25,.5,1\}$, the target fraction of tasks to repeat. Seven
requested mixed-task fractions up to 20\% and two ways to assign positive
counts to mixed tasks give 236 settings. The counts are either 1 and 4
or 2 and 3, each used equally often. For LiveCodeBench, Hull uses the
same observations as Joint and a conservative bound based on the
distribution of the number of positive labels sampled within each task.
For each target width, we find the first tested budget whose mean width
meets it. The 19 optimization problems minimize the worst expected
width over pure cohorts, requiring coverage on every cohort and keeping
each sampling design fixed. These diagnostics use data already inspected.

\subsection{Longer paths and historical comparisons}
\label{app:kgt1_response}\label{app:empirical}
The $K=1000$ replay, planned after inspecting the data, uses both halves of
22 recorded outcome grids of
\citet{brown2024}, five disjoint paths per task per half, both decisive events,
six budgets, and 1000 randomizations per design and cell. At $K=2$, the
agent replay likewise reuses inspected outcomes for 12 configurations of \citet{bjarnason2026}, both
events, six budgets, and 500 randomizations per design and cell. At both $K$
values, Joint and exact uniform reserve $(M+t)K$ and charge each sampled path
through the first decisive event or $K$, whichever comes first. The extended
appendix contains the complete grids and historical comparisons.

\end{document}